\documentclass{article}
\usepackage{amsmath,amssymb,amsbsy,amsthm,algorithmic,algorithm}

\usepackage{nips12submit_e,times}

\usepackage{graphicx} % more modern

\usepackage{capt-of}

\usepackage{hyperref}

\let\epsilon\varepsilon
\let\phi\varphi

\def\X{{\cal X}}
\def\N{\mathbb N}

\def\R{\mathbb R}

\def\H{\mathcal H}
\def\bH{\mathbf H}
\def\F{\mathcal F}
\def\E{{\bf E}}

\def\-as{\text{-a.s.}}

\newtheorem{theorem}{Theorem}
\newtheorem{definition}{Definition}
\newtheorem{lemma}{Lemma}

\newenvironment{remark}[1][Remark.]{\begin{trivlist}
\item[\hskip \labelsep {\bfseries #1}]}{\end{trivlist}}

\newtheorem{claim}{Claim}

\begin{document}
\title{Reducing statistical time-series problems to  binary classification}
\author{
Daniil Ryabko 
\\SequeL-INRIA/LIFL-CNRS, \\Universit\'e de Lille, France\\
 \texttt{daniil@ryabko.net} 
\And
 J\'er\'emie Mary
\\SequeL-INRIA/LIFL-CNRS, \\Universit\'e de Lille, France\\ \texttt{Jeremie.Mary@inria.fr}
}
%\address{SequeL-INRIA/LIFL-CNRS, Universit\'e de Lille, France}
\nipsfinalcopy
\maketitle

\begin{abstract}
We  show how binary classification methods developed to work on i.i.d.\ data can be 
used for solving  statistical problems that are seemingly unrelated to classification and concern highly-dependent time series. 
Specifically, the problems of time-series  clustering, homogeneity testing and the three-sample problem  are addressed. The algorithms that we construct for solving 
these problems are based on a new metric between time-series distributions, which can be evaluated using binary classification methods. 
Universal consistency of the  proposed algorithms  is proven under most general assumptions. The theoretical results are illustrated with experiments on synthetic and real-world data.
\end{abstract}
\section{Introduction}
Binary classification is one of the most well-understood problems of machine learning and statistics:  a wealth of efficient classification algorithms 
has been developed and applied to a wide range of applications. Perhaps one of the reasons for this is that binary classification is conceptually
one of the simplest statistical learning  problems.  It is thus  natural 
to try and use it as a building block for solving other, more complex, newer  or just different problems;
in other words, one can try to obtain efficient algorithms for different learning problems by  reducing them to binary classification.
This approach has been applied to many different problems, starting with multi-class classification, and including
regression and ranking \cite{Balcan:07,Langford:06}, to give just a few examples.
 %but also including some  statistical problems such as 
%homogeneity testing and change-point detection (see, e.g., \cite{Kifer:04}). 
	 However, all of these problems are formulated in terms of 
independent and identically distributed (i.i.d.) samples.  This is also the assumption underlying the theoretical analysis 
of most of the classification algorithms.

In this work we consider learning problems that concern time-series  data for which  independence assumptions do not hold. The 
series can exhibit arbitrary long-range dependence, and different time-series samples may be interdependent as well.
Moreover, the learning problems that we consider~--- the three-sample problem, time-series clustering, and homogeneity testing~---  at first glance seem completely unrelated to classification.

We show how the considered problems can be reduced to  binary classification methods.
The results include   asymptotically  consistent algorithms,   as well as
finite-sample analysis.
To establish the consistency of the suggested methods, for clustering and the three-sample problem the only
assumption  that we make on the data  is that the distributions generating the samples are stationary ergodic; this is one of 
 the weakest assumptions used in statistics. For homogeneity testing we have to make some mixing assumptions 
in order to obtain consistency results (this is indeed unavoidable \cite{Ryabko:10discr}). Mixing conditions are also 
used to obtain finite-sample performance guarantees for the first two problems.

The proposed approach is based on a new distance between time-series distributions (that is, between probability 
distributions on the space of infinite sequences), which we call {\em telescope distance}. This distance can be evaluated using binary classification  methods, and its
finite-sample estimates are shown to be asymptotically consistent. 
Three main building blocks are used to construct the telescope distance. The first one is  a distance on finite-dimensional 
marginal distributions. The  distance  we use for this is the following: $d_\H(P,Q):=\sup_{h\in\H} |\E_P h- \E_Q h|$ where $P,Q$ are distributions and $\H$ is a set of functions. This distance 
can be estimated  using  binary classification methods, and thus can be used to reduce various statistical problems to the classification problem.
This distance  was previously applied to such statistical problems as homogeneity testing and change-point estimation \cite{Kifer:04}.
However, these applications so far have only concerned i.i.d.\ data, whereas we want to work with highly-dependent time series.
Thus, the second  building block are the recent results of \cite{Adams:10, Adams:12}, that show that empirical estimates of $d_\H$ are consistent (under
certain conditions on $\H$) for arbitrary stationary ergodic distributions.  This, however, is not enough: evaluating $d_\H$ for (stationary 
ergodic) time-series distributions means  measuring the distance
between their finite-dimensional marginals, and not the distributions themselves. Finally, the third step to construct the distance is what we call {\em telescoping}.
It consists in summing the distances for all the (infinitely many) finite-dimensional marginals with decreasing weights. 
% The resulting distance can ``automatically''
% select the marginal distribution of the right order: marginals which can not distinguish between the distributions will give distance estimates
% that converge to zero, while marginals whose orders are too high to have converged 

We show that the resulting distance (telescope distance)  indeed can be consistently estimated based on sampling, for arbitrary stationary 
ergodic distributions. Further, we show how this fact can be used to construct consistent algorithms for the considered problems on time series.
Thus we can harness binary classification methods to solve statistical learning problems concerning time series.
% While the main results of this work are theoretical, we argue that all the proposed methods are easily computable, and readily available 
% classification methods (such as SVMs) can be utilized to this end. Finally, we also provide some topological analysis of the 
% telescope distance proposed in this work as compared to the distributional distance, that was previously used for statistical analysis of time series.
% Namely, we show that the telescope distance is stronger.
 
To illustrate the theoretical results in an experimental setting, we chose the problem of  time-series clustering, since it is a difficult unsupervised  problem which seems most different
from the problem of binary classification. Experiments on  both synthetic and real-world data are provided. %As a classifier we use a SVM.
The real-world setting concerns  brain-computer interface (BCI) data, which is  a notoriously  challenging application, and on which the presented  algorithm  
demonstrates competitive performance.
%Our algorithms demonstrate competitive performance in a completely  unsupervised setting

A related  approach to address the problems considered here, as well some related problems about stationary ergodic 
time series, is based on  (consistent) empirical estimates of the distributional  distance, see \cite{Ryabko:103s,Ryabko:10clust,Khaleghi:12} and \cite{Gray:88} about the distributional distance. 
The empirical distance is based on counting frequencies of bins of decreasing sizes and ``telescoping.''  A similar telescoping trick is used in different problems, e.g.\ sequence prediction \cite{BRyabko:88}.
  Another related approach to time-series analysis involves a different reduction, namely, that 
  to data compression \cite{BRyabko:09}.  

%The rest of the paper is organized as follows.
{\bf Organisation.}
 Section~\ref{s:def} is preliminary. In Section~\ref{s:dist} we 
introduce and discuss the telescope distance. Section~\ref{s:red} explains how 
this distance  can be calculated using binary classification methods. Sections~\ref{s:tsc} and~\ref{s:clust}
are devoted to the three-sample problem and clustering, respectively.
In Section~\ref{s:speed}, under some  mixing conditions, we address the problems of homogeneity testing, clustering with unknown $k$, and finite-sample performance guarantees.
Section~\ref{s:exp} presents experimental evaluation. %Some proofs are deferred to the appendix. % Proofs of obvious statements are  omitted.
% of the approach on synthetic and real-world time-series clustering problems.
 % Section~\ref{s:comp} discusses computational issues. Finally, 
%Section~\ref{s:topo} provides some topological analysis of the telescope distance. %, and  Section~\ref{s:cncl} concludes.

\section{Notation and definitions}\label{s:def}
Let $(\X,\mathcal F_1)$ be a measurable space (the domain), and denote $(\X^k,\mathcal F_k)$ and $(\X^\N,\mathcal F)$ the product probability space over $\X^k$ and the induced
probability space over the one-way infinite  sequences taking values in  $\X$. 
Time-series (or process) distributions  are probability measures on the space $(X^\N,\mathcal F)$. 
%  of one-way infinite sequences (where $\mathcal F_{\N}$ is the Borel
%  sigma-algebra of $X^\N$).
%The notation below is chosen so that to make the treatment of the discrete- and real-valued cases maximally similar. 
We use the abbreviation $X_{1..k}$ for $X_1,\dots,X_k$.  %All sets and functions introduced below (in particular, the sets $\H_k$ and their elements) are  assumed measurable.
A set $\H$ of  functions is called  {\em separable} if there is a countable set $\H'$ of functions such that any function in $\H$ is a pointwise
limit of a sequence of elements of $\H'$.

A distribution  $\rho$ is stationary if $\rho(X_{1..k}\in A)=\rho(X_{n+1..n+k}\in A)$ for all $A\in\mathcal F_{k}$, $k,n\in\N$. % (with $\mathcal F_{\X^k}$ being the sigma-algebra of $X^k$).
A stationary distribution is called (stationary) ergodic if $\lim_{n\to\infty}{1\over n}\sum_{i=1..n-k+1}\mathbb I_{X_{i..i+k}\in A}=\rho(A)$ $\rho$-a.s.
for every $A\in\mathcal F_{k}$, $k\in\N$. (This definition, which is more suited for the purposes of this work, is equivalent to the 
usual one expressed in terms of invariant sets, see, e.g., \cite{Gray:88}.)
\section{A distance between time-series distributions}\label{s:dist}
We start with a  distance between distributions on $\X$, and then we will extend it to distributions on $\X^\N$.
For two probability  distributions $P$ and  $Q$~on $(\X,\F_1)$ and  a set $\H$ of measurable functions on $\X$, one can define the  distance
$$
d_\H(P,Q):=\sup_{h\in\H} |\E_P h- \E_Q h|.
$$
This metric has been studied since at least \cite{Zolotarev:76}; 
its special cases include Kolmogorov-Smirnov  \cite{Kolmogorov:33}, Kantorovich-Rubinstein \cite{Kantorovich:57} and Fortet-Mourier \cite{Fortet:53} metrics.
Note that the distance function so defined may not be measurable; however, it is measurable under mild conditions which we assume when necessary.
In particular, separability of $\H$ is a sufficient condition (separability is required in most of the results below).

We will be interested in the cases where $d_\H(P,Q)=0$ implies $P=Q$. Note that in this case $d_\H$ is  a metric (the rest 
of the properties are easy to see).
%  Note that the triangle inequality trivially holds (for all $\H$),
% and the question is really only about whether $d_\H(P,Q)=0$ implies $P=Q$.
For reasons that will become apparent shortly (see Remark below), we will be mainly interested in  the
 sets $\H$ that  consist of indicator functions. In this case we can identify each $f\in\H$ with 
the indicator set $\{x:f(x)=1\}\subset \X$ and (by  a slight abuse of notation) write 
$
d_\H(P,Q):=\sup_{h\in\H} |P(h)- Q(h)|.
$ 
In this case it  is easy to check that the following statement holds true. 
\begin{lemma}
 $d_\H$ is a metric on the space of probability distributions over $\X$ if 
 and only if $\H$ generates~$\F_1$.
\end{lemma}
 The  property that $\H$ generates~$\F_1$ is often easy to verify directly. First of all, it trivially holds
for the case where $\H$ is the set of halfspaces in  a Euclidean $\X$. It is also easy to check that it holds if $\H$ is the set of halfspaces 
in the feature space of most commonly  used kernels (provided the feature space is of the same or higher dimension than the input space), 
such as polynomial  and Gaussian kernels.

Based on $d_\H$ we can construct a distance between time-series probability distributions.
For two time-series distributions $\rho_1,\rho_2$ we take the  $d_\H$ between $k$-dimensional marginal distributions
of $\rho_1$ and $\rho_2$ for each $k\in\N$, and sum them all up with decreasing weights.

\begin{definition}[telescope distance $D_\bH$]\label{d:tele}
 For two time series distributions $\rho_1$ and $\rho_2$ on the space $(\X^\N,\F)$ 
and a  sequence of sets of functions $\bH=(\H_1,\H_2,\dots)$ %(all measurable)
 define the {\em telescope distance} 
\begin{equation}\label{eq:ts}
 D_\bH(\rho_1,\rho_2)%\\
:=\sum_{k=1}^\infty w_k \sup_{h\in\H_k} |\E_{\rho_1} h (X_1,\dots,X_k)- \E_{\rho_2} h(Y_1,\dots,Y_k)|,
\end{equation}
where $w_k$, $k\in\N$ is  a sequence of positive summable real weights (e.g., $w_k=1/k^2$ or $w_k=2^{-k}$).
\end{definition}

\begin{lemma}\label{th:m}
$ D_\bH$ is a metric if and only if  $d_{\H_k}$ is a metric for every $k\in\N$.  
\end{lemma}
\begin{proof}
 The  statement follows from the fact that two process distributions are the same if and only if all their finite-dimensional marginals coincide.
\end{proof}
%Next we introduce empirical estimates of the telescope distance.
\begin{definition}[empirical telescope distance $\hat D$]\label{d:etele}
 For a pair of samples $X_{1..n}$ and $Y_{1..m}$ define {\em empirical telescope distance} 
as 
\begin{multline}\label{eq:ets}
 \hat D_\bH(X_{1..n},Y_{1..m}):=\\
\sum_{k=1}^{\min\{m,n\}} w_k \sup_{h\in\H_k} \left|{1\over n-k+1}\sum_{i=1}^{n-k+1}  h (X_{i..i+k-1})- {1\over m-k+1}\sum_{i=1}^{m-k+1}  h(Y_{i..i+k-1})\right|.
\end{multline}
%where $\gamma_n, n\in\N$ is a non-decreasing sequence of positive numbers $\gamma_n\to\infty$.
\end{definition}

All the methods presented in this work are based on the empirical telescope distance. The key fact is that it is an asymptotically consistent 
estimate of the telescope distance, that is, the latter can be consistently estimated based on sampling.

\begin{theorem}\label{th:cons}
 Let $\bH=(\H_k)_{k\in\N}$  be a sequence  of separable sets $\H_k$ of  indicator functions (over $\X^k$)
of finite VC dimension such that $\H_k$ generates $\F_{k}$. Then, for every stationary ergodic time series distributions $\rho_X$ and $\rho_Y$ generating
samples $X_{1..n}$ and $Y_{1..m}$ we have
\begin{equation}\label{eq:cons}
 \lim_{n,m\to\infty} \hat D_\bH(X_{1..n},Y_{1..m})= D_\bH(\rho_X,\rho_Y)
\end{equation}
\end{theorem}
%The proof is deferred to the appendix. %, relies on the fact \cite{Adams:12} that 
%
Note that $\hat D_{\bf H}$ is a biased estimate of $D_{\bf H}$, and, unlike in the i.i.d.\ case, the bias
may depend on the distributions; however, the bias is $o(n)$.

\begin{remark}\label{rem:ind} The condition that the sets $\H_k$ are sets of indicator function of finite VC dimension comes from \cite{Adams:12}, where 
it is shown that for any stationary ergodic distribution $\rho$, under these conditions,   $\sup_{h\in\H_k} {1\over n-k+1}\sum_{i=1}^{n-k+1}  h (X_{i..i+k-1})$ is an asymptotically
consistent estimate of $\sup_{h\in\H_k} \E_{\rho} h (X_1,\dots,X_k)$. This fact implies  that $d_{\H_k}$ can be consistently estimated, from which the theorem is derived. 
\end{remark}

\begin{proof}[Proof of Theorem~\ref{th:cons}]
 As  established in \cite{Adams:12},
 under the conditions of the theorem we have
\begin{equation}\label{eq:ada}
 \lim_{n\to\infty}\sup_{h\in\H_k} {1\over n-k+1}\sum_{i=1}^{n-k+1}  h (X_{i..i+k-1})=  \sup_{h\in\H_k} \E_{\rho_X} h (X_1,\dots,X_k)\text{ $\rho_X$-a.s.}
\end{equation}
for all $k\in\N$, and likewise for $\rho_Y$. 
Fix an $\epsilon>0$. We can find a $T\in\N$ such that 
\begin{equation}\label{eq:t}
\sum_{k>T}w_k\le \epsilon.
\end{equation}
Note that $T$ depends only on $\epsilon$.
Moreover, as follows from~(\ref{eq:ada}), for each $k=1..T$ we can find an $N_k$ such that 
\begin{equation}\label{eq:ada3}
 \Big|\sup_{h\in\H_k} {1\over n-k+1}\sum_{i=1}^{n-k+1}  h (X_{i..i+k-1}) %\\
-\sup_{h\in\H_k} \E_{\rho_X} h (X_{1..k})\Big| \le   \epsilon/T
\end{equation}
Let $N_k:=\max_{i=1..T}N_i$ and define analogously $M$ for~$\rho_Y$.
Thus, for $n\ge N$, $m\ge M$ we have
\begin{multline*}
\hat D_\bH(X_{1..n},Y_{1..m})\\
\le \sum_{k=1}^{T} w_k \sup_{h\in\H_k} \left|{1\over n-k+1}\sum_{i=1}^{n-k+1}  h (X_{i..i+k-1})
%\right. \\\left.
- {1\over m-k+1}\sum_{i=1}^{m-k+1}  h(Y_{i..i+k-1})\right|+\epsilon
\\\le \sum_{k=1}^{T} w_k \sup_{h\in\H_k}\Bigg\{ \left|{1\over n-k+1}\sum_{i=1}^{n-k+1}  h (X_{i..i+k-1})
%\right.\\\left.
 - \E_{\rho_1} h (X_{1..k})\right|
\\
+|\E_{\rho_1} h (X_{1..k})- \E_{\rho_2} h(Y_{1..k})|
\\
+
\left|\E_{\rho_2} h(Y_{1..k})- {1\over m-k+1}\sum_{i=1}^{m-k+1}  h(Y_{i..i+k-1})\right| \Bigg\}+\epsilon
\\\le 3\epsilon + D_{\bf H}(\rho_X,\rho_Y),
\end{multline*} 
where the first inequality follows from the definition~ \eqref{eq:ets} of $\hat D_\bH$  
and from~(\ref{eq:t}), and the last inequality follows from~(\ref{eq:ada3}).
Since $\epsilon$ was chosen arbitrary the statement follows.
\end{proof}

\section{Calculating $\hat D_\bH$ using binary classification methods}\label{s:red}
The methods for solving various statistical problems that we suggest are all based on $\hat D_\bH$. 
The main appeal of this approach is that $\hat D_\bH$ can be calculated using binary classification methods.
Here we explain how to do it.

The definition~(\ref{eq:ets}) of $D_\bH$ involves calculating $l$ summands (where $l:=\min\{n,m\}$), that is 
\begin{equation}\label{eq:sup}
 \sup_{h\in\H_k} \left|{1\over n-k+1}\sum_{i=1}^{n-k+1}  h (X_{i..i+k-1}) - {1\over m-k+1}\sum_{i=1}^{m-k+1}  h(Y_{i..i+k-1})\right|
\end{equation}
for each $k=1..l$. Assuming that $h\in\H_k$ are indicator functions,
calculating each of the summands amounts to solving the following $k$-dimensional binary classification problem.
Consider $X_{i..i+k-1}$, $i=1 .. n-k+1$  as class-1 examples and   $Y_{i..i+k-1}$, $i=1..m-k+1$ as class-0 examples.
The supremum~(\ref{eq:sup}) is attained on $h\in\H_k$ that minimizes the  empirical risk, with examples weighted with respect to the sample size.
 Indeed, %if $\H_k$ is a set of  indicator functions on $\X^k$
 we can define the weighted  empirical risk of any $h\in\H_k$ as 
$$
 \left|{1\over n-k+1}\sum_{i=1}^{n-k+1} (1- h (X_{i..i+k-1})) + {1\over m-k+1}\sum_{i=1}^{m-k+1} h(Y_{i..i+k-1})\right|,
$$
which is obviously minimized by any $h\in\H_k$ that attains~(\ref{eq:sup}).

Thus, as long as we have a way to find $h\in\H_k$ that minimizes empirical risk, we have a consistent estimate of $D_\H(\rho_X,\rho_Y)$, under 
the mild conditions on $\bH$ required by Theorem~\ref{th:cons}. Since the dimension of the resulting classification problems 
 grows with the length of the sequences, one should prefer methods that work in high dimensions, such as soft-margin SVMs \cite{Cortes:95}.

A particularly remarkable feature is that {\em the choice of $\H_k$ is  much easier} for the problems that we consider 
{\em than in the binary classification} problem. Specifically, if (for some fixed $k$) the classifier that achieves the minimal (Bayes) error
for the classification problem is not in $\H_k$, then obviously the error of an empirical risk  minimizer will not tend to zero, no matter
how much data we have.  In contrast, all we need to achieve asymptotically 0 error in estimating $\hat D$ (and therefore, in the 
learning problems considered below) is that the sets $\H_k$  generate $\mathcal F_{k}$ and have a finite VC dimension (for each $k$).
 This is the case already for the set of half-spaces in $\R_k$. In other words, the {\em approximation} error  of the binary classification method 
(the classification error of the best $f$ in $\H_k$)
is not important. What is important is the estimation error; for asymptotic consistency results it has to go to 0 (hence the requirement on the VC dimension);
for non-asymptotic results, it will appear in the error bounds, see Section~\ref{s:speed}. 
Thus, we have the following statement.
\begin{claim} The approximation error $|D_\bH(P,Q)-\hat D_\bH(X,Y)|$, and thus the error of the algorithms below, can be much smaller than the error
 of classification algorithms used to calculate $D_\bH(X,Y)$.
\end{claim}

% Note that  the choice of $H_k$ (or, say, of the kernel to use in SVM) is still important from the practical point of view, it is almost irrelevant 
% for the theoretical consistency results. 
We can conclude that, beyond the requirement that $\H_k$  generate $\mathcal F_{k}$ for each $k\in\N$, the choice of $H_k$ (or, say, of the kernel to use in SVM) is
entirely up to the needs and constraints of specific applications.

Finally, we remark that while in the definition of the empirical distributional distance~(\ref{eq:ets}) the number of summands is $l$ (the length
of the shorter of the two samples),  it can be replaced with any $\gamma_{l}$ such that $\gamma_{l}\to\infty$, without affecting any asymptotic consistency results.
In other words, Theorem~\ref{th:cons}, as well as all the consistency statements below, hold true for $l$ replaced with any function $\gamma_l$ that increases to infinity. 
A practically viable choice is $\gamma_l=\log l$; in fact, there is no reason to choose faster growing $\gamma_n$ since the estimates
for higher-order summands will not have enough data to converge. % (which, again, has no impact on the consistency properties of $\hat D$).
This is also the value we use in the experiments.  % (Section~\ref{s:exp}).

\section{The  three-sample problem}\label{s:tsc}
We start with a conceptually simple problem known in statistics as the three-sample problem (some times also called time-series classification).
We are given three samples $X=(X_1,\dots,X_n)$, $Y=(Y_1,\dots,Y_m)$ and  $Z=(Z_1,\dots,Z_l)$. It is known
that $X$ and $Y$ were generated by different time-series distributions,  whereas  $Z$ was generated by the same 
distribution as either $X$ or $Y$. It is required to find out which one is the case. Both distributions are assumed
to be stationary ergodic, but no further assumptions are made about them (no independence, mixing or memory assumptions).
%
% Note that this problem is very different from the traditional problem of (binary) classification of i.i.d.\ samples.
% Indeed, in the latter problem one typically assumes that a large number of (i.i.d.) training samples from each of the two
% classes is given, and possibly a large number of (i.i.d.) testing samples has to be classified. In contrast, in the-time
% series classification problem we are given just two training points~--- one from each class~--- and one testing point  to classify.
%
The three sample-problem for dependent time series has been addressed in \cite{Gutman:89} for Markov processes and in \cite{Ryabko:103s} for stationary 
ergodic time series. 
The latter work uses an approach based on the  distributional distance. %, whose empirical estimates are based on counting  frequencies. 

Indeed, to solve this problem it  suffices to have consistent estimates of some distance between time series distributions. Thus,
we can use  the telescope  distance. The following statement is a simple corollary of Theorem~\ref{th:cons}.
\begin{theorem}\label{th:cl}
 Let the samples  $X=(X_1,\dots,X_n)$, $Y=(Y_1,\dots,Y_m)$ and  $Z=(Z_1,\dots,Z_l)$ be generated by stationary ergodic 
distributions $\rho_X, \rho_Y$ and $\rho_Z$, with $\rho_X\ne\rho_Y$ and either (i) $\rho_Z=\rho_X$ or (ii) $\rho_Z=\rho_Y$.
Let the sets  $\H_k$, $k\in\N$  be  separable sets of  indicator functions over $\X^k$.
Assume that each set $\H_k$, $k\in\N$ has a finite VC dimension  and  generates $\F_{k}$.
A test that declares that (i) is true if $\hat D_{\bf H}(Z,X)\le \hat D_{\bf H}(Z,Y)$ and that (ii) is true otherwise, makes only finitely many errors with probability~1
as $n,m,l\to\infty$.
\end{theorem}
It is straightforward to extend this theorem to more than two classes; in
other words, instead of $X$ and $Y$ one can have an arbitrary number of
samples 
%generated by
from
different stationary ergodic distributions.
A further generalization of this problem is the problem of time-series clustering, considered in the next section.

\section{Clustering time series}\label{s:clust}
We are given $N$ time-series samples  $X^1=(X_1^1,\dots,X_{n_1}^1),\dots,X^N=(X_1^N,\dots,X_{n_{N}}^N)$, and it is required 
to cluster them into $K$ groups, where, in different settings,  $K$ may be either known or unknown.
While there may be many different approaches to define what should be considered a good clustering, and, thus, what  it means to 
have a consistent clustering algorithm, 
 for the problem of clustering time-series samples there is a natural choice, proposed in \cite{Ryabko:10clust}:
 Assume that each of the  time-series samples  $X^1=(X_1^1,\dots,X_{n_1}^1),\dots,X^N=(X_1^N,\dots,X_{n_{N}}^N)$ was generated by one out of $K$
different   time-series distributions $\rho_1,\dots,\rho_K$. 
These distributions are unknown.
The  {\em target clustering} is defined according to whether the samples were generated by the same or different distributions:
 the samples belong to the same cluster if and only if 
they were generated by the same distribution.
A clustering algorithm is called {\em asymptotically consistent} if with probability 1  from some $n$ on it outputs the 
target clustering, where $n$ is the length of the shortest sample 
  $n:=\min_{i=1..N}n_i\ge n'$. 

Again, to solve this problem it is enough to have a metric between time-series distributions that can be consistently estimated. 
Our approach here is based on the telescope distance, and thus we use~$\hat D$. 

The clustering problem is relatively simple if the target clustering has what is called the {\em strict separation property} \cite{Balcan:08}:
every two points in the same   target cluster %(generated by the same distribution) 
are closer to each other than to any point from a different  target cluster. % (generated by different distributions).
 The following statement is an easy corollary of Theorem~\ref{th:cons}.

\begin{theorem}\label{th:ss}
Let the sets  $\H_k$, $k\in\N$  be  separable sets of  indicator functions over $\X^k$.
Assume that each set $\H_k$, $k\in\N$ has a finite VC dimension  and  generates $\F_{k}$.
% Assume that the sets  $\H_k\subset\X^k$, $k\in\N$  are  separable sets of  indicator functions
% of finite VC dimension,  such that $\H_k$ generates $\F_{k}$.
 If the distributions  $\rho_1,\dots,\rho_K$ generating the samples  $X^1=(X_1^1,\dots,X_{n_1}^1),\dots,X^N=(X_1^N,\dots,X_{n_{N}}^N)$
are stationary ergodic, then with probability 1 from some $n:=\min_{i=1..N}n_i$ on the target clustering has the strict separation property with respect to~$\hat D_{\bf H}$.
\end{theorem}

With the strict separation property at hand, if the number of clusters $K$ is known, it is easy to find asymptotically consistent algorithms. Here we  give some simple examples, but the theorem
below can be extended to many other distance-based clustering algorithms.

The {\em average linkage} algorithm works as follows. The distance between clusters 
is defined as the average distance between  points in these clusters.  First, put each point into a separate cluster. Then, merge the two 
closest clusters; repeat the last step until the total number of clusters is $K$. 
The {\em farthest point} clustering works as follows. Assign $c_1:=X^1$ to the first cluster. For $i=2..K$,
find the point $X^j$, $j\in\{1..N\}$ that maximizes the distance $\min_{t=1..i}\hat D_{\bf H}(X^j,c_t)$ (to the points already assigned to clusters) and assign $c_i:=X^j$ to
the cluster $i$.  Then assign each of the remaining points to the nearest cluster.
The following statement is a corollary of Theorem~\ref{th:ss}.

\begin{theorem}\label{th:clt}
Under the conditions of Theorem~\ref{th:ss}, average linkage and farthest point clusterings are asymptotically consistent, provided the correct number of clusters $K$ is given to the algorithm. 
\end{theorem}
Note that we do not require the samples to be independent; the joint distributions of the samples may be completely arbitrary, as long 
as the marginal distribution of each sample is stationary ergodic. These results can be extended to the online setting in the spirit of \cite{Khaleghi:12}.

For the case of unknown number of clusters, the situation is different: one has to make stronger assumptions on the distributions generating the samples,
since there is no algorithm that is consistent for all stationary ergodic distributions \cite{Ryabko:10discr}; such stronger  assumptions are considered in the next section. % and the discussion in the next section).

\section{Speed of convergence}\label{s:speed}
The results established so far are asymptotic out of necessity: they are established under the assumption
that the distributions involved are stationary ergodic,  which is too general 
to allow for   any meaningful finite-time performance guarantees.
While it is interesting to be able to establish consistency results under such general assumptions,
 it is also interesting to see what results can be  obtained under  stronger assumptions.
Moreover, since it is usually not known in advance whether the data at hand satisfies given assumptions or not, 
it appears important to have methods that have {\em both} asymptotic consistency in the general setting and finite-time performance
guarantees under stronger assumptions. It turns out that this is possible: for the methods based on $\hat D$ one can establish 
both the asymptotic performance guarantees for all stationary ergodic distributions and finite-sample performance guarantees under 
stronger assumptions, namely the uniform mixing conditions introduced below.

Another reason to consider stronger assumptions on the distributions generating the data is that  some statistical problems, such as homogeneity testing or clustering when 
the number of clusters is unknown, are provably impossible to solve under the only assumption of stationary ergodic distributions, as shown in~\cite{Ryabko:10discr}.

Thus, in this section we analyse the speed of convergence of $\hat D$ under certain mixing conditions, and use 
it to construct solutions for the problems of homogeneity and clustering with an unknown number of clusters, as well 
as to establish finite-time performance guarantees for the methods presented in the previous sections.

%\subsection{$\beta$-mixing}
 A stationary distribution on the space of one-way infinite sequences $(\X^\N,\mathcal F)$ can be uniquely extended to a stationary 
 distribution on the space of two-way infinite sequences $(\X^{\mathbb Z},\mathcal F_{\mathbb Z})$ of the form $\dots,X_{-1},X_0,X_1,\dots$.
\begin{definition}[$\beta$-mixing coefficients]
For a process distribution $\rho$ define the mixing coefficients
$$
 \beta(\rho,k):=\sup_{\substack{A\in \sigma(X_{-\infty..0}),\\ B\in\sigma(X_{k..\infty})}} |\rho(A\cap B)-\rho(A)\rho(B)|
$$ where $\sigma(..)$ denotes the sigma-algebra of the random variables in brackets.
\end{definition}
When $\beta(\rho,k)\to0$ the process $\rho$ is called  uniformly $\beta$-mixing
(with coefficients $\beta(\rho,k)$); this condition is much stronger than ergodicity, but is much weaker than
the i.i.d.\ assumption. %For more information on mixing the reader is referred to~\cite{Bosq:96}.

\subsection{Speed of convergence of $\hat D$}
Assume that a sample $X_{1..n}$ is generated by a distribution $\rho$ that is uniformly $\beta$-mixing
with coefficients $\beta(\rho,k)$ %bounded by $\beta_t, t\in\N$.
 Assume further that $\H_k$ is a set of indicator functions with 
a finite VC dimension $d_k$, for each $k\in\N$.

Since in this section we are after finite-time bounds, we fix a concrete choice of the weights $w_k$ in the definition~\ref{d:tele} of $\hat D$,
\begin{equation}\label{eq:wk}
 w_k:=2^{-k}.
\end{equation}

 The general tool that we use to obtain performance guarantees in this section is the following bound
that can be obtained from the results of \cite{Karandikar:02}. 
\begin{multline}\label{eq:mixtl}
 q_n(\rho,\H_k,\epsilon)
%\\
:= \rho\left(\sup_{h\in\H_k} \left|{1\over n-k+1}\sum_{i=1}^{n-k+1}  h (X_{i..i+k-1})
%\\
-\E_{\rho} h (X_{1..k})\right| >\epsilon\right)\\\le n\beta(\rho,t_n-k)+8t_n^{d_k+1}e^{-l_n\epsilon^2/8},
\end{multline}
where $t_n$ are  any integers in $1..n$ and $l_n=n/t_n$. %, $\gamma:=\exp(-\epsilon^2/32)$ and $C(\epsilon):=8
The parameters $t_n$ should be set according to the values of $\beta$ in order to optimize the bound.

One can use similar bounds for classes of finite Pollard dimension \cite{Pollard:84} or more general bounds expressed in terms of covering
numbers, such as those given in \cite{Karandikar:02}.  Here we consider classes of finite VC dimension only 
 for the ease of the exposition and for the sake of continuity with the previous section (where it was necessary).

Furthermore, for the rest of this section we assume geometric $\beta$-mixing distributions, that is, $\beta(\rho,t)\le\gamma^t$ for some $\gamma<1$.
Letting $\l_n=t_n=\sqrt{n}$ the bound~(\ref{eq:mixtl}) becomes
\begin{equation}\label{eq:mix}
 q_n(\rho,\H_k,\epsilon)\le n\gamma^{\sqrt{n}-k}+8n^{(d_k+1)/2}e^{-\sqrt{n}\epsilon^2/8}.
\end{equation}

\begin{lemma}\label{th:mix}
Let two  samples $X_{1..n}$ and $Y_{1..m}$ be generated by stationary distributions $\rho_X$ and $\rho_Y$ whose $\beta$-mixing
coefficients satisfy $\beta(\rho_{.},t)\le\gamma^t$ for some $\gamma<1$. Let $\H_k$, $k\in\N$ be some sets of indicator functions on $\X^k$ whose
VC dimension $d_k$ is finite and non-decreasing with $k$. Then
\begin{equation}\label{eq:speed}
P( |\hat D_\bH(X_{1..{n}},Y_{1..{m}})- D_\bH(\rho_X,\rho_Y)|>\epsilon)\le 2\Delta(\epsilon/4,n')
\end{equation}
where $n':=\min\{n,m\}$, the probability is with respect to $\rho_X\times\rho_Y$ and 
\begin{equation}\label{eq:delt}
 \Delta(\epsilon,n)%\\
:=-\log\epsilon( n\gamma^{\sqrt{n}+\log(\epsilon)}+8n^{(d_{-\log\epsilon}+1)/2}e^{-\sqrt{n}\epsilon^2/8}).
\end{equation}
\end{lemma}
\begin{proof}%[Proof of Lemma~\ref{th:mix}]
 From~\eqref{eq:wk} we have $\sum_{k=-\log\epsilon/2}^\infty w_k<\epsilon/2$. Using this and the definitions~(1) and (2) %\ref{d:tele} and~\ref{d:etele} 
of $D_{\bf H}$ and $\hat D_{\bf H}$ we obtain
\begin{multline*}
 P( |\hat D_\bH(X_{1..{n_1}},Y_{1..{n_2}})- D_\bH(\rho_X,\rho_Y)|>\epsilon)\\\leq \sum_{k=1}^{-\log( \epsilon/2)}( q_n(\rho_X,\H_k,\epsilon/4)+q_n(\rho_Y,\H_k,\epsilon/4)),
\end{multline*}
which, together with~(6) %(\ref{eq:mix}),
 implies the statement.
\end{proof}

\subsection{Homogeneity testing}
Given two samples $X_{1..n}$ and $Y_{1..m}$ generated by distributions $\rho_X$ and $\rho_Y$ respectively, the problem of homogeneity testing 
(or the two-sample problem) consists in deciding whether $\rho_X=\rho_Y$.  
A test is called (asymptotically) consistent if its probability of error goes to zero as $n':=\min\{m,n\}$ goes to infinity.
As mentioned above, in general, for stationary ergodic time series distributions there is no asymptotically consistent test for homogeneity \cite{Ryabko:10discr} (even 
for binary-valued time series); thus,  stronger assumptions are in order.

Homogeneity testing is one of the classical problems of mathematical statistics, and one of the most studied ones. Vast
literature exits on homogeneity testing for i.i.d.\ data, and for dependent processes as well. We do not attempt to survey this literature here.
Our contribution to this line of research is to show that this problem can be reduced 
 (via the telescope distance) to binary classification, in the case of strongly dependent processes satisfying some mixing conditions.

It is easy to see that under the mixing conditions of Lemma~1 a consistent test for homogeneity exists, and finite-sample
performance guarantees can be obtained. It is enough to find a sequence $\epsilon_n\to0$ such that $\Delta(\epsilon_n,n)\to0$
(see \eqref{eq:delt}). 
Then the test can be constructed as follows: say that the two sequences $X_{1..n}$ and $Y_{1..m}$ were generated by the same distribution 
if $\hat D_\bH(X_{1..n},Y_{1..m})<\epsilon_{\min\{n,m\}}$; otherwise say that they were generated by  different distributions.

\begin{theorem}\label{th:hom} Under the conditions of Lemma~\ref{th:mix} the probability of Type~I error (the distributions are the same but the test
says they are different) of the described test is upper-bounded by $2\Delta(\epsilon/4,n')$. The probability of Type~II error (the distributions
are different but the test says they are the same) is upper-bounded by $2\Delta((\delta-\epsilon)/4,n')$ where $\delta:=D_{\bf H}(\rho_X,\rho_Y)$.
\end{theorem}
\begin{proof}
 The  statement is an immediate consequence of Lemma~\ref{th:mix}. Indeed, for the Type~I error, the two sequences are generated
by the same distribution, so the probability of error of the test is given by~\eqref{eq:speed} with $D_\bH(\rho_X,\rho_Y)=0$. The probability of Type~II error is 
given by $P(D_\bH(\rho_X,\rho_Y)- \hat D_\bH(X_{1..{n_1}},Y_{1..{n_2}})>\delta-\epsilon)$, which is upper-bounded by $2\Delta((\delta-\epsilon))/4,n')$ as follows from~\eqref{eq:speed}.
\end{proof}

The optimal choice of $\epsilon_n$ may depend on the speed at which $d_k$ (the VC dimension of $\H_k$) increases; however,
for most natural cases (recall that $\H_k$ are also parameters of the algorithm) 
this growth is polynomial, so the main term to control is $e^{-\sqrt{n}\epsilon^2/8}$.

For example, if $\H_k$ is the set of halfspaces in $\X^k=\R^k$ then $d_k= k+1$ and one can chose $\epsilon_n:=n^{-1/8}$.
The resulting probability of Type~I error decreases as $\exp(-n^{1/4})$.

\subsection{Clustering with a known or unknown number of clusters}
If the distributions generating the samples satisfy certain mixing conditions, then we can augment Theorems~\ref{th:ss} and~\ref{th:clt} with 
finite-sample performance guarantees. 
\begin{theorem}
Let the distributions  $\rho_1,\dots,\rho_k$ generating the samples  $X^1=(X_1^1,\dots,X_{n_1}^1),\dots,X^N=(X_1^N,\dots,X_{n_{N}}^N)$ 
satisfy the conditions of Lemma~\ref{th:mix}.  Define $\delta:=\min_{i,j=1..N, i\ne j}D_{\bf H}(\rho_i,\rho_j)$ and $n:=\min_{i=1..N} n_i$.
Then with probability at least $$1-N(N-1)\Delta(\delta/12,n')$$ the target 
clustering of the samples has the strict separation property. In this case single linkage and farthest point algorithms output
the target clustering.
\end{theorem}
\begin{proof}
 Note that a sufficient condition for the strict separation property to hold is that for every pair $i,j$ of samples generated by the same distribution we have $\hat D_{\bf H}(X^i,X^j)\le\delta/3$,
and for every  pair $i,j$ of samples generated by different distributions we have $\hat D_{\bf H}(X^i,X^j)\ge2\delta/3$. 
Using Lemma~\ref{th:mix},  the probability of such an even  (for each pair) is upper-bounded  by $2\Delta(\delta/12,n')$, which, multiplied
by the total number  $N(N-1)/2$ of pairs gives the statement. The second
statement is obvious.
\end{proof}

As with homogeneity testing, while in the general case of  stationary ergodic distributions it is impossible to have a consistent clustering algorithm when the number 
of clusters $k$ is unknown, the situation changes if the distributions satisfy certain mixing conditions.
%Indeed, it is easy to see that   there exists a consistent clustering algorithm in this case, with unknown~$k$.
In this case a consistent clustering algorithm can be  obtained as follows. Assign to the same cluster all samples that are at most $\epsilon_n$-far from each
other, where the threshold $\epsilon_n$ is selected the same way as for homogeneity testing:  $\epsilon_n\to0$ and $\Delta(\epsilon_n,n)\to0$.
The optimal choice of this parameter depends on the choice of $\H_k$ through the speed of growth of the VC dimension $d_k$ of these sets.

\begin{theorem}
Given $N$ samples  generated by $k$ different stationary distributions $\rho_i$, $i=1..k$ (unknown $k$) all satisfying the conditions of  Lemma~\ref{th:mix},
 the probability of error (misclustering at least one sample) of the described algorithm is upper-bounded 
by $$N(N-1)\max\{\Delta(\epsilon/4,n'),\Delta((\delta-\epsilon)/4,n')\}$$ where $\delta:=\min_{i,j=1..k,i\ne j}D_{\bf H}(\rho_i,\rho_j)$ and $n=\min_{i=1..N} n_i$, with $n_i$, $i=1..N$ being lengths of the samples.
\end{theorem}
\begin{proof}
 The statement follows from Theorem~\ref{th:hom}.
\end{proof}

\section{Experiments}\label{s:exp}
For experimental evaluation we chose the problem of time-series clustering. Average-linkage clustering is used, 
with the telescope distance between samples calculated using an SVM, as described in Section~\ref{s:red}. In all experiments, SVM is used with radial basis kernel, with default parameters of libsvm \cite{Chang:11}.
The parameters $w_k$ in the definition of the telescope distance (Definition~\ref{d:tele}) are set to $w_k:=k^{-2}$.

\subsection{Synthetic data}
For the artificial setting we have chosen  highly-dependent time series distributions which have  the same single-dimensional marginals
and which cannot be well approximated by finite- or countable-state models. %The data generation process is as follows.
The distributions $\rho(\alpha)$, $\alpha\in(0,1)$, are constructed as follows.
Select  $r_0 \in [0,1]$ uniformly at random; then, for each $i=1..n$ obtain $r_i$ by shifting
$r_{i-1}$ by $\alpha$ to the right, and removing the integer part. %, i.e. 
%$r_i := r_{i-1}+ \alpha -\lfloor r_{i-1}+ \alpha \rfloor$.
The time series $(X_1,X_2,\dots)$ is then obtained from $r_i$ by drawing
a point from a distribution law $\mathcal N_1$ if $r_i<0.5$ and from  $\mathcal N_2$ otherwise.
% We call this procedure $CAS(\alpha,\mathcal N_1, \mathcal N_2)$. 
$\mathcal N_1$ is a 3-dimensional  Gaussian
with mean of 0 and  covariance matrix $\operatorname{Id}\times 1/4$. $\mathcal N_2$ is the same but with  mean  $1$.
If $\alpha$ is irrational\footnote{in experiments simulated by a \texttt{longdouble} with a long mantissa}
 then the distribution $\rho(\alpha)$  is stationary ergodic, but does not belong to any  simpler natural
distribution family \cite{Sheilds:96}.
The single-dimensional marginal is the same for all values of $\alpha$. The latter two properties make all parametric and most non-parametric methods 
inapplicable to this problem.

In our experiments, we use two  process distributions  
 $\rho(\alpha_i), i\in\{1,2\}$, with 
$\alpha_1 = 0.31...,~\alpha_2 = 0.35...,$.   The
dependence of  error rate on the length of time series is shown on Figure~\ref{fig:errlen}. 
One clustering experiment on sequences of length 1000 takes about 5 min.\ on a standard laptop.
\subsection{Real data}
To demonstrate the applicability of the proposed methods to realistic scenarios, we chose
the brain-computer interface data from BCI competition III \cite{Milan:04}.
The dataset consists of (pre-processed) BCI recordings of mental imagery:
a person is thinking about one  of three subjects (left foot,  right foot,  a  random letter).
Originally, each time series consisted of several consecutive sequences of different classes,
and the problem was supervised: three time series for training and one for testing.
We split each of the original time series into classes, and then used our clustering algorithm in a completely unsupervised 
setting.  The original problem is 96-dimensional, but we used only  the first 3 dimensions (using all 96 gives worse performance). The typical sequence length is 300.
 The performance is reported in Table~\ref{fig:errlen}, labeled $\operatorname{TS_{SVM}}$. All the computation for this experiment takes approximately 6 minutes on a standard laptop.

The following methods were used for comparison. First, we used   dynamic time wrapping (DTW) \cite{Sakoe:78}
which is a popular base-line approach for time-series clustering. 
The other two methods  in Table~1 are from \cite{Harchaoui:08}. The comparison
is not fully relevant, since the results in \cite{Harchaoui:08} are for different settings;
the method KCpA was used in change-point estimation method (a different but also unsupervised setting), and SVM 
was used in a supervised setting. The latter is of particular interest since the classification method 
we used in the telescope distance is also SVM, but our setting is unsupervised (clustering).

\begin{figure}[!htbp]

\begin{minipage}{6.8cm}
\begin{centering}
\vspace{-0.4cm}
\includegraphics[width=6.5cm]{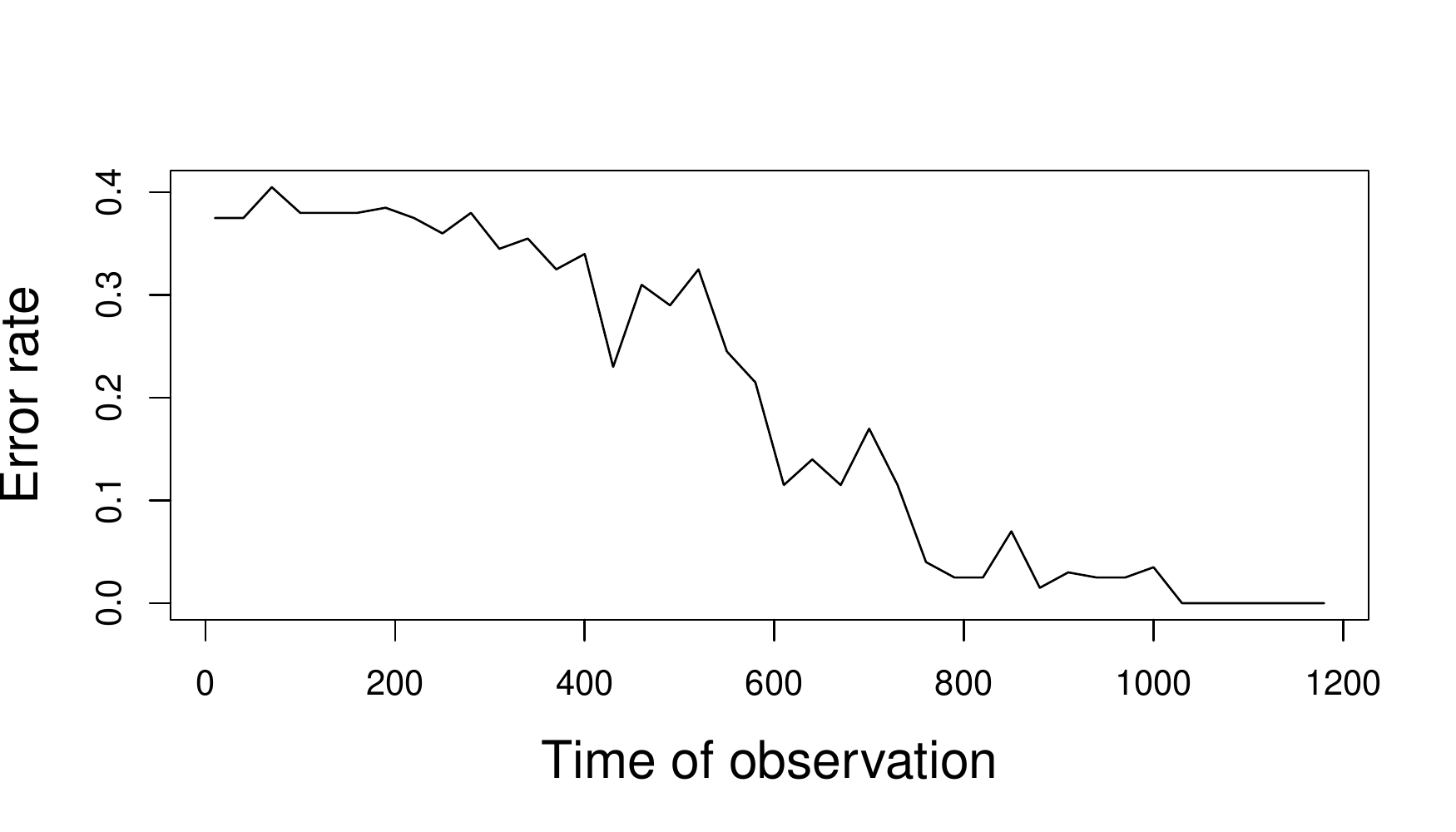}\vspace{-0.2cm}
\vspace{-0.1cm}
\captionof{figure}{Error of two-class clustering using $\operatorname{TS_{SVM}}$; 10 time series in each target cluster, averaged over 20 runs.}
\label{fig:errlen} 
\end{centering}
%\end{figure}
\end{minipage}
\hspace{0.1cm}
\begin{minipage}{6.5cm}
%\begin{table}
\begin{centering}
\begin{tabular}{|l|c|c|c|}
\hline
~ & $s_1$ & $s_2$ & $s_3$ \\
\hline
$\operatorname{TS_{SVM}}$ & \textbf{84\%}       & \textbf{81\%}       & \textbf{61\%}       \\ 
DTW            & 46\%       & 41\%       & 36\%       \\ 
KCpA   & 79\%       & 74\%       & 61\%       \\ 
SVM   & 76\%       & 69\%       & 60\%       \\ 
\hline
\end{tabular}
\captionof{table}{Clustering accuracy in the BCI dataset. 3 subjects (columns), 4 methods (rows). Our method is $\operatorname{TS_{SVM}}$.}
\label{tab:acc} 
\end{centering}
%\end{table}
\end{minipage}

\end{figure}

{\small {\bf Acknowledgments.} This research 
was funded by the Ministry of Higher Education and Research, Nord-Pas-de-Calais Regional Council and 
 FEDER (Contrat de Projets Etat Region CPER  2007-2013),
  ANR  projects EXPLO-RA (ANR-08-COSI-004),
  Lampada (ANR-09-EMER-007)  and CoAdapt,
  and by the European Community's  FP7 Program  under grant agreements 
 n$^\circ$\,216886 (PASCAL2) and n$^\circ$\,270327 (CompLACS).}


\begin{thebibliography}{10}

\bibitem{Adams:10}
Terrence~M. Adams and Andrew~B. Nobel.
\newblock Uniform convergence of {V}apnik-{C}hervonenkis classes under ergodic
  sampling.
\newblock {\em The Annals of Probability}, 38:1345--1367, 2010.

\bibitem{Adams:12}
Terrence~M. Adams and Andrew~B. Nobel.
\newblock Uniform approximation of {V}apnik-{C}hervonenkis classes.
\newblock {\em Bernoulli}, 18(4):1310--1319, 2012.

\bibitem{Balcan:07}
Maria-Florina Balcan, Nikhil Bansal, Alina Beygelzimer, Don Coppersmith, John
  Langford, and Gregory Sorkin.
\newblock Robust reductions from ranking to classification.
\newblock In Nader Bshouty and Claudio Gentile, editors, {\em Learning Theory},
  volume 4539 of {\em Lecture Notes in Computer Science}, pages 604--619. 2007.

\bibitem{Balcan:08}
M.F. Balcan, A.~Blum, and S.~Vempala.
\newblock A discriminative framework for clustering via similarity functions.
\newblock In {\em Proceedings of the 40th annual ACM symposium on Theory of
  computing}, pages 671--680. ACM, 2008.

\bibitem{Chang:11}
Chih-Chung Chang and Chih-Jen Lin.
\newblock {LIBSVM}: A library for support vector machines.
\newblock {\em ACM Transactions on Intelligent Systems and Technology},
  2:27:1--27:27, 2011.
\newblock Software available at \url{http://www.csie.ntu.edu.tw/~cjlin/libsvm}.

\bibitem{Cortes:95}
Corinna Cortes and Vladimir Vapnik.
\newblock Support-vector networks.
\newblock {\em Mach. Learn.}, 20(3):273--297, 1995.

\bibitem{Fortet:53}
R.~Fortet and E.~Mourier.
\newblock {Convergence de la r\'epartition empirique vers la r\'epartition
  th\'eoretique.}
\newblock {\em Ann. Sci. Ec. Norm. Super., III. Ser}, 70(3):267--285, 1953.

\bibitem{Gray:88}
R.~Gray.
\newblock {\em Probability, Random Processes, and Ergodic Properties}.
\newblock Springer Verlag, 1988.

\bibitem{Gutman:89}
M.~Gutman.
\newblock Asymptotically optimal classification for multiple tests with
  empirically observed statistics.
\newblock {\em IEEE Transactions on Information Theory}, 35(2):402--408, 1989.

\bibitem{Harchaoui:08}
Za\"{\i}d Harchaoui, Francis Bach, and Eric Moulines.
\newblock Kernel change-point analysis.
\newblock In {\em NIPS}, pages 609--616, 2008.

\bibitem{Kantorovich:57}
L.~V. Kantorovich and G.~S. Rubinstein.
\newblock On a function space in certain extremal problems.
\newblock {\em Dokl. Akad. Nauk USSR}, 115(6):1058--1061, 1957.

\bibitem{Karandikar:02}
R.L. Karandikar and M.~Vidyasagar.
\newblock Rates of uniform convergence of empirical means with mixing
  processes.
\newblock {\em Statistics and Probability Letters}, 58:297--307, 2002.

\bibitem{Khaleghi:12}
A.~Khaleghi, D.~Ryabko, J.~Mary, and P.~Preux.
\newblock Online clustering of processes.
\newblock In {\em AISTATS}, JMLR W\&CP 22, pages 601--609, 2012.

\bibitem{Kifer:04}
Daniel Kifer, Shai Ben-David, and Johannes Gehrke.
\newblock Detecting change in data streams.
\newblock In {\em Proceedings of the Thirtieth international conference on Very
  large data bases - Volume 30}, VLDB'04, pages 180--191, 2004.

\bibitem{Kolmogorov:33}
A.N. Kolmogorov.
\newblock Sulla determinazione empirica di una legge di distribuzione.
\newblock {\em G. Inst. Ital. Attuari}, pages 83--91, 1933.

\bibitem{Langford:06}
John Langford, Roberto Oliveira, and Bianca Zadrozny.
\newblock Predicting conditional quantiles via reduction to classification.
\newblock In {\em UAI}, 2006.

\bibitem{Milan:04}
Jos{\'{e}} del~R. Mill{\'{a}}n.
\newblock On the need for on-line learning in brain-computer interfaces.
\newblock In {\em Proc. of the Int. Joint Conf. on Neural Networks}, 2004.

\bibitem{Pollard:84}
D.~Pollard.
\newblock {\em Convergence of Stochastic Processes}.
\newblock Springer, 1984.

\bibitem{BRyabko:88}
B.~Ryabko.
\newblock Prediction of random sequences and universal coding.
\newblock {\em Problems of Information Transmission}, 24:87--96, 1988.

\bibitem{BRyabko:09}
B.~Ryabko.
\newblock Compression-based methods for nonparametric prediction and estimation
  of some characteristics of time series.
\newblock {\em IEEE Transactions on Information Theory}, 55:4309--4315, 2009.

\bibitem{Ryabko:10clust}
D.~Ryabko.
\newblock Clustering processes.
\newblock In {\em Proc. the 27th International Conference on Machine Learning
  (ICML 2010)}, pages 919--926, Haifa, Israel, 2010.

\bibitem{Ryabko:10discr}
D.~Ryabko.
\newblock Discrimination between {B}-processes is impossible.
\newblock {\em Journal of Theoretical Probability}, 23(2):565--575, 2010.

\bibitem{Ryabko:103s}
D.~Ryabko and B.~Ryabko.
\newblock Nonparametric statistical inference for ergodic processes.
\newblock {\em IEEE Transactions on Information Theory}, 56(3):1430--1435,
  2010.

\bibitem{Sakoe:78}
H.~Sakoe and S.~Chiba.
\newblock Dynamic programming algorithm optimization for spoken word
  recognition.
\newblock {\em IEEE Transactions on Acoustics, Speech and Signal Processing},
  26(1):43--49, 1978.

\bibitem{Sheilds:96}
P.~Shields.
\newblock {\em The Ergodic Theory of Discrete Sample Paths}.
\newblock AMS Bookstore, 1996.

\bibitem{Zolotarev:76}
V.~M. Zolotarev.
\newblock Metric distances in spaces of random variables and their
  distributions.
\newblock {\em Math. USSR-Sb}, 30(3):373--401, 1976.

\end{thebibliography}
\end{document}